\def\year{2020}\relax
\documentclass[letterpaper]{article} 
\usepackage{aaai20}  
\usepackage{times}  
\usepackage{helvet} 
\usepackage{courier}  
\usepackage[hyphens]{url}  
\usepackage{graphicx} 
\usepackage{graphicx}  
\usepackage{soul}
\usepackage{amsmath}
\usepackage{xspace}
\usepackage{mathtools}
\usepackage{amsthm}
\usepackage{amssymb}

\def\tt{true}
\def\ff{false}

\newtheorem{definition}{Definition}
\newtheorem{theorem}{Theorem}
\newtheorem{lemma}{Lemma}

\newcommand{\LTL}{{\sc ltl}\xspace} 
\newcommand{\LTLf}{{\sc ltl}$_f$\xspace} 
\newcommand{\Prop}{\mathcal{P}} 
\newcommand{\G}{\mathcal{G}} 
\newcommand{\X}{\mathcal{X}} 
\newcommand{\Y}{\mathcal{Y}} 
\newcommand{\DFA}{{\sc dfa}\xspace} 
\newcommand{\T}{\mathcal{T}} 
\newcommand{\limp}{\mathbin{\rightarrow}} 

\title{\LTLf Synthesis with Fairness and Stability Assumptions}
\author{
Shufang Zhu,\textsuperscript{\rm 1}
Giuseppe De Giacomo,\textsuperscript{\rm 2}
Geguang Pu,\textsuperscript{\rm 1}
Moshe Y. Vardi\textsuperscript{\rm 3}\\ 
\textsuperscript{\rm 1} Each China Normal University,  
\textsuperscript{\rm 2} Sapienza Universit\`a di Roma,
\textsuperscript{\rm 3} Rice University\\
shufangzhu.szhu@gmail.com, degiacomo@diag.uniroma1.it, ggpu@sei.ecnu.edu.cn, vardi@cs.rice.edu 
}

\begin{document}
\maketitle
\begin{abstract}
	In synthesis, assumptions are constraints on the environment that rule out certain environment behaviors. A key observation here is that even if we consider systems with \LTLf goals on finite traces, environment assumptions need to be expressed over infinite traces, since accomplishing the agent goals may require an unbounded number of environment action.
	To solve synthesis with respect to finite-trace \LTLf goals under infinite-trace assumptions, we could reduce the problem to \LTL synthesis. Unfortunately, while synthesis in \LTLf and in \LTL have the same worst-case complexity~(both 2EXPTIME-complete), the algorithms available for \LTL synthesis are much more difficult in practice than those for \LTLf synthesis.  
	In this work we show that in interesting cases we can avoid such a detour to \LTL synthesis and keep the simplicity of \LTLf synthesis. Specifically, we develop a BDD-based fixpoint-based technique for handling basic forms of fairness and of stability assumptions. We show, empirically, that this technique performs much better than standard \LTL synthesis.
\end{abstract}

\section{Introduction}

In many situations we are interested in expressing properties over an unbounded but finite sequence of successive states. Linear-time Temporal Logic over finite traces~(\LTLf) and its variants have been thoroughly investigated for doing so. 
There has been broad research for logical reasoning~\cite{DV13,LiRPZV19}, synthesis~\cite{DegVa15,CamachoBMM18}, and planning~\cite{CTMBM17,DegRu18}.

Recently synthesis under assumptions in \LTLf has attracted specific interest~\cite{DegRu18,CamachoBM18}.
First, planning for \LTLf goals can be considered as a form of \LTLf synthesis under assumptions, where the assumptions are the dynamics of the environment encoded in the planning domain \cite{Gree69,CamachoBM18,AminofGMR18,ADMRicaps19}.
However, more generally, assumptions can be arbitrary constraints on the environment that can be exploited by the agent in devising a strategy to fulfill its goal.

Synthesis under assumptions has been extensively studied in \LTL, where environment assumptions are expressed as \LTL formulas~\cite{ChatterjeeH07,ChatterjeeHJ08,DIppolitoBPU13,BloemEK15,BrenguierRS17}. In fact, \LTL formulas can be used as assumptions as long as it is guaranteed that the environment is able to behave so as to keep the assumptions true, i.e., assumptions are environment realizable. Under these circumstances, it is possible
to reduce synthesis for \LTL goal $\psi_G$ under assumptions $\psi_A$ to standard synthesis for $\psi_A \limp \psi_G$. Note that because of the guarantee of $\psi_A$ being environment realizable, no agent strategy can win $\psi_A \limp \psi_G$ by falsifying $\psi_A$. See~\cite{ADMRicaps19} for a discussion.

When we turn to \LTLf, a key observation is that even if we consider (finite-trace) \LTLf goals for the agent, assumptions need to be expressed considering infinite traces, since accomplishing the agent goals may require an unbounded number of environment action. So we have an assumption $\psi_A $ expressed in \LTL and a goal $\phi_G$ expressed in \LTLf. To solve synthesis under
assumptions in \LTLf, we could translate $\phi_G$ into \LTL getting $\psi_G$, by applying the translation of \LTLf into \LTL in~\cite{DV13}, and then do \LTL synthesis for $\psi_A \limp \psi_G$, see e.g.~\cite{CamachoBM18}.

Unfortunately, while synthesis in \LTLf and in \LTL have the same worst-case complexity, being both 2EXPTIME-complete~\cite{PnueliR89,DegVa15}, the algorithms available for \LTL synthesis are much harder in practice than those for \LTLf synthesis. In particular, the lack of efficient algorithms for the crucial step of automata determinization is prohibitive for finding scalable implementations~\cite{DFogartyKVW13,finkbeiner2016synthesis}. In spite of recent advancement in synthesis such as reducing to parity games~\cite{MeyerSL18}, bounded synthesis based on solving iterated safety games~\cite{KupfermanV05,FinkbeinerS13,GerstackerKF18}, or recent techniques based on iterated FOND planning~\cite{CamachoBMM18}, \LTL synthesis remains very challenging. In contrast, \LTLf synthesis is based on a translation to Deterministic Finite Automaton~(\DFA)~\cite{RaSc59}, which can be seen as a game arena where environment and agent make their own moves. On this arena, the agent wins if a simple fixpoint condition~(reachability of the \DFA accepting states) is satisfied.

Hence, when we introduce assumptions, an important question arises: can we retain the simplicity of \LTLf synthesis? In particular,  we are thinking about algorithms based on devising some sort of arena and then extracting winning strategies by relying on computing a small number of nested fixpoints~(note that the reduction of \LTL synthesis to parity games may generate exponentially many nested fixpoints~\cite{ALG02bis}).

We consider here two different basic, but quite significant, forms of assumptions: a basic form of \emph{fairness} $GF\alpha$~(always eventually $\alpha$), and a basic form of \emph{stability} $FG\alpha$~(eventually always $\alpha$), where in both cases the truth value of  $\alpha$ is under the control of the environment, and hence the assumptions are trivially realizable by the environment. Note that due to the existence of \LTLf goals, synthesis under both kinds of assumptions does not fall under known easy forms of synthesis, such as GR(1) formulas~\cite{BloemJPPS12}. For these kinds of assumptions, we devise a specific algorithm based on using the \DFA for the \LTLf goal as the arena and then computing 2-nested fixpoint properties over such arena. It should be noted that the kind of nested fixpoint that we compute for \emph{fairness} $GF\alpha$ is similar to the one in~\cite{DegRu18}, but it is clear that the ``fairness" stated there is different from what we claim in this paper. The ``fairness" in~\cite{DegRu18} is interpreted as all effects happening fairly, therefore the assumption is hardcoded in the arena itself. Here, instead, we only require that a selected condition $\alpha$ happens fairly, and our technique extends to deal with \emph{stability} assumptions as well. We compare the new algorithm with standard \LTL synthesis~\cite{MeyerSL18} and show empirically that this algorithm performs significantly better, in the sense that solving more cases with less time cost. Some proofs have been removed due to the lack of space.\footnote{A full version is available on arXiv. Geguang Pu is the corresponding author.}

\section{Preliminaries}
\textit{Linear-time Temporal Logic over finite traces} (\LTLf) has the same syntax as \LTL over infinite traces introduced in~\cite{Pnu77}. Given a set of propositions $\Prop$, the syntax of \LTLf formulas is defined as 	$\phi ::= 
a\ |\ \neg \phi\ |\ \phi_1\wedge\phi_2\ |\ X\phi\ |\ \phi_1 U \phi_2$.
Every $a\in \Prop$ is an \textit{atom}. A literal $l$ is an atom or the negation of an atom. $X$ for ``Next", and $U$ for ``Until",
are temporal operators. 
We make use of the standard Boolean abbreviations, such as $\vee$~(or) and $\rightarrow$~(implies), $\tt$ and $\ff$. Additionally, we define the following abbreviations ``Weak Next" $X_w \phi\equiv \neg X\neg \phi$, ``Eventually" $F\phi\equiv \tt U\phi$ and ``Always" $G\phi\equiv \ff R\phi$, where $R$ is for ``Release". 

A \textit{trace} $\rho = \rho[0],\rho[1],\ldots$ is a sequence of propositional interpretations (sets), where 
$\rho[m]\in 2^\Prop$ ($m \geq 0$) is the $m$-th interpretation of $\rho$, and $|\rho|$ represents the length of $\rho$. Intuitively, $\rho[m]$ is interpreted as the set of propositions which are $true$ at instant $m$.
Trace $\rho$ is an \textit{infinite} trace if $|\rho| = \infty$, which is formally denoted as $\rho\in (2^\Prop)^{\omega}$; 
otherwise $\rho$ is a \textit{finite} trace, denoted as $\rho\in (2^\Prop)^{*}$. \LTLf formulas are interpreted over finite, nonempty traces. Given a finite trace $\rho$ and an \LTLf formula 
$\phi$, we inductively define when $\phi$ is $true$ on $\rho$ at step $i$ ($0 \leq i < |\rho|$), written $\rho, i \models \phi$, as follows: 

\noindent
$\bullet~\rho, i \models a$ iff $a \in \rho[i]$; \\
$\bullet~\rho, i \models \neg \phi$ iff $\rho,i \not\models \phi$; \\
$\bullet~\rho, i \models\phi_1 \wedge \phi_2$ iff $\rho,i \models \phi_1$ and $\rho, i \models \phi_2$; \\
$\bullet~\rho, i \models X\phi$ iff $i+1 < |\rho|$ and $\rho, i+1 \models \phi$; \\
$\bullet~\rho, i \models \phi_1 U \phi_2$ iff there exists $j$ such that $i\leq j < |\rho|$ and $\rho, j\models \phi_2$, and for all $k$, $i \leq k < j$, we have $\rho, k \models \phi_1$.

An \LTLf formula $\phi$ is $true$ on $\rho$, denoted by $\rho \models \phi$, if and only if $\rho, 0\models\phi$. 

\textit{\LTLf synthesis} can be viewed as a game of two players, the \textit{environment} and the \textit{agent}, contrasting each other. The aim is to synthesize a strategy for the agent such that no matter how the environment behaves, the combined behavior trace of both players satisfy the logical specification expressed in \LTLf~\cite{DegVa15}.

\section{Fair and Stable \LTLf Synthesis}
In this paper, we focus on \LTLf synthesis under assumptions in two different basic forms: fairness and stability, which we call in the following~\emph{fair \LTLf synthesis} and \emph{stable \LTLf synthesis}, respectively. In such synthesis problems, both players~(environment and agent) have Boolean variables under their respective control. Here, we use $\X$ to denote the set of environment variables that are uncontrollable for the agent, and $\Y$ the set of agent variables that are controllable for the agent. Therefore, $\X$ and $\Y$ are disjoint.

In general, assumptions are specific forms of constraints.
\begin{definition}[Environment Constraint]
	An environment constraint  $\alpha$ is a Boolean formula over $\X$.
\end{definition}
In particular, we define here two different basic, but common forms of assumptions.
\begin{definition}[Fairness and Stability Assumptions]
	An \LTL formula $\psi$ is considered as a fairness assumption if it is of the form $GF\alpha$, and a stability assumption if of the form $FG\alpha$, where in both cases $\alpha$ is an environment constraint.
\end{definition}
A fair or stable trace can then be defined in terms of the corresponding assumption~(fairness or stability).

\begin{definition}[Fair and Stable Traces]
	A trace $\rho \in (2^{\X \cup \Y})^\omega$ is $\alpha$-fair if $\rho \models GF \alpha$ and it is $\alpha$-stable if  $\rho \models FG \alpha$.
\end{definition}
Intuitively, $\alpha$ holds infinitely often on an $\alpha$-fair trace, while eventually holds forever on an $\alpha$-stable trace.
Note that, if trace $\rho$ is not $\alpha$-fair, i.e., $\rho \nvDash GF\alpha$, then $\rho \models  FG(\neg \alpha)$ such that $\rho$ is $\neg \alpha$-stable. Similarly, if trace $\rho$ is not $\alpha$-stable, i.e., $\rho \nvDash FG\alpha$, then $\rho \models  GF(\neg \alpha)$ such that $\rho$ is $\neg \alpha$-fair. Although there is a duality between fairness and stability, such duality breaks when applying these environment assumptions to the problem of \LTLf synthesis. This is because in addition to the assumptions, the synthesis problems also require the \LTLf specification to be satisfied.

We now define fair and stable \LTLf synthesis by making use of fair and stable traces.

\begin{definition}[Fair~(Stable) \LTLf Synthesis]
	\LTLf formula $\phi$, defined over $\X \cup \Y$, is $\alpha$-fair~(resp., $\alpha$-stable) realizable if there exists a strategy $g: (2^{\X})^+\rightarrow 2^{\Y}$, such that for an arbitrary environment trace $\lambda = X_0,X_1,\ldots\in (2^{\X})^{\omega}$, if $\lambda$ is $\alpha$-fair~(resp., $\alpha$-stable), then we can find $k\geq 0$ such that $\phi$ is $true$ in the finite trace $\rho^k = (X_0\cup g(X_0)), (X_1\cup g(X_0,X_1)), \ldots, (X_k\cup g(X_0,X_1,\ldots,X_{k}))$. 
	
	A fair~(resp., stable) \LTLf synthesis problem, described as a tuple $\langle \X, \Y, \alpha, \phi \rangle$, consist in checking whether $\phi$, defined over $\X \cup \Y$, is $\alpha$-fair~(resp., $\alpha$-stable) realizable.	The synthesis procedure aims to computing a strategy if realizable.
\end{definition}

Intuitively speaking, $\phi$ describes the desired goal when the environment behaviors satisfy the assumption.
An agent strategy $g:~(2^{\X})^+\rightarrow 2^{\Y}$ for fair~(resp., stable) synthesis problem $\langle \X, \Y, \alpha, \phi \rangle$ is \emph{winning} if it guarantees the satisfaction of the objective $\phi$ under the condition that the environment behaves in a way that $\alpha$ holds infinitely often~(resp., $\alpha$ eventually holds forever). The \emph{realizability procedure} of $\langle \X, \Y, \alpha, \phi \rangle$ aims to answer the existence of a winning strategy $g$ and the \emph{synthesis procedure} amounts for computing $g$ if it exists.
In fact one can consider two variants of the synthesis problem, depending on the player who moves first, in the sense of assigning values to variables under its control first. Here we consider the environment as the first-player (as in planning), but a version where the agent moves first can be obtained by a small modification. 

Since every \LTLf formula $\phi$ can be translated to a Deterministic Finite Automaton~(\DFA) $\G_\phi$ that accepts exactly the same language as $\phi$~\cite{DV13}, we are able to reduce the problems of fair \LTLf synthesis and stable \LTLf synthesis to specific two-player \DFA games, in particular, fair \DFA game and stable \DFA game, respectively. We start with introducing \DFA games.

\subsection{Games over \DFA}
Two-player games on \DFA are games consisting of two players, the \emph{environment} and the \emph{agent}. $\X$ and $\Y$ are disjoint sets of environment Boolean variables and agent Boolean variables, respectively.
The \emph{specification} of the game arena is given by a \DFA $\G$ = $(2^{\X\cup \Y}, S, s_0, \delta, Acc)$, where
$2^{\X\cup \Y}$ is the alphabet, $S$ is a set of states, $s_0\in S$ is an initial state, $\delta: S\times 2^{\X\cup \Y}\rightarrow S$ is a transition function and $Acc\subseteq S$ is a set of accepting states.

A \emph{round} of the game consists of both players setting the values of variables under their respective control. A \emph{play} $\rho$ over $\G$ records how two players set the values at each round and how the \DFA proceed according to the values. Formally, a play $\rho$ from state $s_{i_0}$ is an infinite trace $(s_{i_0},X_0 \cup Y_0),(s_{i_1},X_1 \cup Y_1)\ldots\in (S\times2^{\X\cup\Y})^\omega$ such that $s_{i_{j+1}}=\delta(s_{i_j},X_j \cup Y_j)$. Moreover, we also assign the \emph{environment} as the first-player, which sets values first.

A play $\rho$ is considered as a winning play if it follows a certain winning condition. Different winning conditions lead to different games. In this paper, we consider two specific two-player games, fair \DFA game and stable \DFA game, both of which are described as $\langle \G, \alpha \rangle$, where $\G$ is the game arena and $\alpha$ is the environment constraint.

\noindent\textbf{Fair DFA Game.}
Although the ultimate goal for solving a fair \DFA game is to perform winning plays for the agent, since it is more straightforward to formulate the game considering the environment as the protagonist, we first define the winning condition of the environment over a play. A play $\rho=(s_{i_0},X_0 \cup Y_0),(s_{i_1},X_1 \cup Y_1)\ldots$ over $\G$ is \emph{winning} for the \emph{environment} with respect to a fair \DFA game $\langle \G, \alpha \rangle$ if the following two conditions hold:

\noindent
$\bullet$~\emph{Recurrence}: $\rho$ is $\alpha$-fair~(that is, $\rho \models GF\alpha$), \\
$\bullet$~\emph{Safety}: $s_{i_j}\not\in Acc$ for all $j\geq 0$~($Acc$ is avoided).

Consequently, a play $\rho$ is \emph{winning} for the \emph{agent} if one of the following conditions holds:

\noindent
$\bullet$~\emph{Stability}: $\rho$ is not $\alpha$-fair~(that is, $\rho \models FG(\neg\alpha)$),\\
$\bullet$~\emph{Reachability}: $s_{i_j}\in Acc$ for some $j\geq 0$~($Acc$ is reached).

\noindent\textbf{Stable DFA Game.}
As for a stable \DFA game $\langle \G, \alpha \rangle$, a play $\rho=(s_{i_0},X_0 \cup Y_0),(s_{i_1},X_1 \cup Y_1)\ldots$ over $\G$ is \emph{winning} for the \emph{environment} if the following two conditions hold:

\noindent
$\bullet$~\emph{Stability}: $\rho$ is $\alpha$-stable~(that is, $\rho \models FG\alpha$),\\
$\bullet$~\emph{Safety}: $s_{i_j}\not\in Acc$ for all $j\geq 0$~($Acc$ is avoided).

Consequently, a play $\rho$ is \emph{winning} for the \emph{agent} if one of the following conditions holds:

\noindent
$\bullet$~\emph{Recurrence}: $\rho$ is not $\alpha$-stable~(that is, $\rho \models GF(\neg\alpha)$), \\
$\bullet$~\emph{Reachability}: $s_{i_j}\in Acc$ for some $j\geq 0$~($Acc$ is reached).

Since we consider here the environment as the first-player, a strategy for the agent is a function $g: (2^{\X})^+ \rightarrow 2^{\Y}$, deciding the values of the controllable variables for every possible history of the uncontrollable variables. Respectively, an environment strategy is a function $h: (2^{\Y})^* \rightarrow 2^{\X}$.
A play $\rho = (s_{i_0},X_0 \cup Y_0),(s_{i_1},X_1 \cup Y_1)\ldots\in (S\times2^{\X\cup\Y})^\omega$ \emph{follows} a strategy $g$ (resp., a strategy $h$), if $Y_j=g(X_0,\ldots,X_{j})$ for all $j\geq 0$ (resp., $X_j=h(Y_0,\ldots,Y_{j-1})$ for all $j > 0$).

We can now define winning states and winning strategies.
\begin{definition}[Winning State and Winning Strategy]\label{def:win_state}
	In the game $\langle \G, \alpha \rangle$ described above, $s \in S$ is a winning state for the agent~(resp., environment) if there exists strategy $g$~(resp., $h$) s.t. every play $\rho$ from $s$ that follows $g$~(resp., $h$) is an agent~(resp., environment) winning play. Then $g$~(resp., $h$) is a winning strategy for the agent~(resp., environment) from $s$.  
	
\end{definition}
As shown in~\cite{Mar75}, both of the fair \DFA game and stable \DFA game described above are \emph{determined}, that is, a state $s \in S$ is a winning state for the agent if and only if $s$ is not a winning state for the environment. The \emph{realizability} procedure of the game consists of checking whether there exists a winning strategy for the agent from initial state $s_0$. The \emph{synthesis} procedure aims to computing such a strategy. 

We then show how to reduce the problems of fair \LTLf synthesis and stable \LTLf synthesis to fair \DFA game and stable \DFA game, respectively. Hence we can solve the \DFA game, thus settling the corresponding synthesis problem.

\section{Solution to Fair \LTLf Synthesis}~\label{sec:fair}
In order to perform fair synthesis on \LTLf, given problem $\langle \X, \Y, \alpha, \phi \rangle$, we first translate the \LTLf specification $\phi$ into a \DFA $\G_\phi$. We then view $\langle \G_\phi, \alpha \rangle$ as a fair \DFA game, and consider exactly the separation between environment and agent variables as in the original synthesis problem. Specifically, we assign $\X$ as the environment variables and $\Y$ as the agent variables. Finally, we solve the fair \DFA game, thus settling the fair \LTLf synthesis problem. The following theorem assesses the correctness of this technique.
\begin{theorem}
	Fair \LTLf synthesis problem $\langle \X, \Y, \alpha, \phi \rangle$ is realizable iff fair \DFA game $\langle \G_\phi, \alpha \rangle$ is realizable.
\end{theorem}
\begin{proof}
	We prove the theorem in both directions.
	
	\noindent
	$\leftarrow: $ 
	Since $\langle \G_\phi, \alpha \rangle$ is realizable for the agent, the initial state $s_0$ is an agent winning state with winning strategy $g: (2^{\X})^+ \rightarrow 2^{\Y}$. Therefore, a play $\rho = (s_0, X_0\cup g(X_0)), (s_1, X_1\cup g(X_0,X_1)), \ldots$ over $\G_\phi$ from $s_0$ following $g$ is a winning play for the agent. Moreover, for every such play $\rho$ from $s_0$, either of the following conditions holds:
	
	\noindent
	$\bullet~\rho \nvDash GF\alpha$ such that $\rho$ is not $\alpha$-fair.\\
	$\bullet~\rho \models GF\alpha$ such that $\rho$ is $\alpha$-fair. Since $\rho$ is winning for the agent, there exists $j \geq 0$ such that $s_j \in Acc$. This implies that $\rho^j \models \phi$ holds, where $ \rho^j = (s_0, X_0\cup g(X_0)), (s_1, X_1\cup g(X_0,X_1)), \ldots , (s_j, X_j\cup g(X_0,X_1,\ldots,X_{j}))$.

	Consequently, the strategy $g$ assures that for an arbitrary environment trace $\lambda = X_0,X_1,\ldots\in (2^{\X})^{\omega}$, if $\lambda$ is $\alpha$-fair, then there is $j\geq 0$ such that $\phi$ is $true$ in the finite trace $\rho^j$. We conclude that $\langle \X, \Y, \alpha, \phi \rangle$ is realizable.
	
	\noindent
	$\rightarrow: $ For this direction, we assume that $\langle \X, \Y, \alpha, \phi \rangle$ is realizable, then there exists a strategy $g: (2^{\X})^+ \rightarrow 2^{\Y}$ that realizes $\phi$. Thus consider an arbitrary environment trace $\lambda \in (2^\X)^\omega$, either of the following conditions holds:
	
	\noindent
	$\bullet~\lambda$ is not $\alpha$-fair, then the induced play $\rho = (s_0, X_0\cup g(X_0)), (s_1, X_1\cup g(X_0,X_1)), \ldots$ over $\G_\phi$ from $s_0$ that follows $g$ is winning for the agent by default.\\
	$\bullet~\lambda$ is $\alpha$-fair, then on the induced play $\rho = (s_0, X_0\cup g(X_0)), (s_1, X_1\cup g(X_0,X_1)), \ldots$ over $\G_\phi$ from $s_0$, there exists $j \geq 0$ such that $\phi$ is $true$ in the finite trace $\rho^j = (s_0, X_0\cup g(X_0)), (s_1, X_1\cup g(X_0,X_1)), \ldots, (s_j, X_j\cup g(X_0,X_1,\ldots,X_{j}))$, in which case $s_j \in Acc$. Therefore, $\rho$ is winning for the agent.
		
	Consequently, we conclude that fair \DFA game $\langle \G_\phi, \alpha \rangle$ is realizable for the agent. 
\end{proof}

\subsection{Fair \DFA Game Solving}
Winning fair \DFA games means that the agent can eventually reach an ``agent wins" region from which if the constraint $\alpha$ holds, then it is possible to reach an accepting state. Given a fair \DFA game $\langle \G, \alpha \rangle$, we proceed as follows: (1) Compute ``agent wins'' region in fair \DFA game $\langle \G, \alpha \rangle$; (2) Check realizability; (3) Return an agent winning strategy if realizable.

Since the environment winning condition is more intuitive, in order to show the solution to fair \DFA game, we start by solving the \emph{Recurrence-Safety} game, which considers the environment as the protagonist. The idea for winning such game is that the environment should remain in an ``environment wins" region from which the constraint $\alpha$ holds infinitely often referring to \emph{Recurrence} game, meanwhile the accepting states are forever avoidable referring to \emph{Safety} game. Therefore, in order to have both of \emph{Recurrence} such that having $GF \alpha$ holds and \emph{Safety} such that avoiding accepting states $s \in Acc$, the ``environment wins" region computation is defined as:  

\begin{centering}
	$Env_{f} = \nu Z.\mu \hat{Z}.(\exists X. \forall Y. ((X \models \alpha \wedge \delta(s, X  \cup Y) \in Z\backslash Acc) \vee \delta(s, X \cup Y) \in \hat{Z}\backslash Acc))$,
\end{centering}

where $X$ ranges over  $2^\X$ and $Y$ over $2^\Y$. 

The fixpoint stages for $Z$~(note $Z_{i+1} \subseteq Z_i$, for $i \geq 0$, by monotonicity) are:

\noindent
$\bullet~Z_0 = S$,\\ 
$\bullet~Z_{i+1} = \mu \hat{Z}.(\exists X. \forall Y. ((X \models \alpha \wedge 
\delta(s, X \cup Y) \in Z_i \backslash Acc) \vee \delta(s, X \cup Y) \in \hat{Z} \backslash Acc))$.

Eventually, $Env_{f} = Z_k$ for some $k$ such that $Z_{k+1} = Z_{k}$.

The fixpoint stages for $\hat{Z}$ with respect to $Z_i$~(note $\hat{Z}_j \subseteq \hat{Z}_{j+1}$, for $j \geq 0$, by monotonicity) are:

\noindent 
$\bullet~\hat{Z}_{i,0} = \emptyset$,\\
$\bullet~\hat{Z}_{i,j+1} = \exists X. \forall Y. ((X \models \alpha \wedge 
\delta(s, X \cup Y) \in Z_i \backslash Acc) \vee \delta(s, X \cup Y) \in \hat{Z}_{i,j} \backslash Acc)$.

Finally, $\hat{Z}_i = \hat{Z}_{i,k}$ for some $k$  such that $\hat{Z}_{i,k+1} = \hat{Z}_{i,k}$.

The following theorem assures that the nested fixpoint computation of $Env_{f}$ collects exactly all environment winning states in fair \DFA game.
\begin{theorem}
	For a fair \DFA game $\langle \G, \alpha \rangle$ and a state $s \in S$, we have $s \in Env_{f}$ iff $s$ is an environment winning state.
\end{theorem}
\begin{proof}
	We prove the two directions separately.
	
	\noindent
	$\leftarrow:$ We prove by showing the contrapositive. If a state $s \notin Env_{f}$, then $s$ must be removed from $Env_{f}$ at stage $i+1$, therefore, $s \in Z_i\backslash Z_{i+1}$. Then $s \notin \mu \hat{Z}.(\exists X. \forall Y. ((X \models \alpha \wedge 
	\delta(s, X \cup Y) \in Z \backslash Acc) \vee \delta(s, X \cup Y) \in \hat{Z} \backslash Acc))$. That is, no matter what the~(environment)~strategy $h$ is, traces from $s$ satisfy \textbf{neither} of the following conditions:
	
	\noindent
	$\bullet~\alpha$ holds and the trace gets trapped in $Z$ without visiting accepting states such that $X \models \alpha \wedge 
	\delta(s, X \cup Y) \in Z \backslash Acc$ holds, in which case $s$ is a new environment winning state;\\
	$\bullet~\alpha$ eventually gets hold and from there we can have $\alpha$ as true infinitely often without visiting accepting states such that $\delta(s, X \cup Y) \in \hat{Z} \backslash Acc$ holds, in which case $s$ is a new environment winning state.
	
	Therefore, $s$ is not an environment winning state. So if $s$ is an environment winning state then $s \in Env_{f}$ holds.
	
	\noindent
	$\rightarrow:$ If a state $s \in Env_{f}$, then $s \in \mu \hat{Z}.(\exists X. \forall Y. ((X \models \alpha \wedge 
	\delta(s, X \cup Y) \in Z \backslash Acc) \vee \delta(s, X \cup Y) \in \hat{Z} \backslash Acc))$. That is, no matter what the~(agent)~strategy $g$ is, traces from $s$ satisfy either of the following conditions:
	
	\noindent
	$\bullet~\alpha$ holds and the trace gets trapped in $Z$ without visiting accepting states such that $X \models \alpha \wedge 
	\delta(s, X \cup Y) \in Z \backslash Acc$ holds, in which case $s$ is a new environment winning state;\\
	$\bullet~\alpha$ eventually gets hold and from there we can have $\alpha$ as true infinitely often without visiting accepting states such that $\delta(s, X \cup Y) \in \hat{Z} \backslash Acc$ holds, in which case $s$ is a new environment winning state.
	
	Thus $s$ is a winning state for the environment. 
\end{proof}

Due to the determinacy of fair \DFA game, the set of agent winning states $Sys_{f}$ can be computed by negating $Env_f$:

\begin{centering}
	$Sys_{f} = \mu Z.\nu \hat{Z}.(\forall X. \exists Y. ((X \models \neg \alpha \vee
	\delta(s, X \cup Y) \in Z \cup Acc) \wedge \delta(s, X \cup Y) \in \hat{Z} \cup Acc))$.
\end{centering}

\begin{theorem}\label{thm:fair_winning_states}
	A fair \DFA game $\langle \G, \alpha \rangle$ has an agent winning strategy if and only if $s_0 \in Sys_{f}$.
\end{theorem}

\subsection{Strategy Extraction}
Having completed the realizability checking procedure, this section deals with the agent winning strategy generation if $\langle \G, \alpha \rangle$ is realizable. It is known that if some strategy that realizes $\phi$ exists, then there also exists a \emph{finite-state strategy} generated by a finite-state \emph{transducer} that realizes $\phi$~\cite{Buchi1990}. Formally,  the agent \emph{winning strategy} $g : (2^\X)^+ \rightarrow 2^\Y$ can be represented as a deterministic finite transducer based on the set $Sys_{f}$, described as below.

\begin{definition}[Deterministic Finite Transducer]
	Given a fair \DFA game $\langle \G, \alpha \rangle$, where $\G = (2^{\X \cup \Y}, S, s_0, \delta, Acc)$, a deterministic finite transducer $\T = (2^\X, 2^\Y, Q, s_0, \varrho, \omega_f)$ of such game is defined as follows: 
	
	\noindent
	$\bullet$~$Q \subseteq S$ is the set of agent winning states s.t. $Q = Sys_{f}$;\\
	$\bullet$~$\varrho : Q \times 2^\X \rightarrow Q$ is the transition function such that $\varrho(q,X) = \delta(q, X\cup Y)$ and $Y = \omega_f(q, X)$;\\
	$\bullet$~$\omega_f: Q \times 2^\X \rightarrow 2^\Y$ is the output function such that at an agent winning state $q$ with assignment $X$, $\omega_f(q,X)$ returns an assignment $Y$ leading to an agent winning play.
\end{definition}

The transducer $\T$ generates $g$ in the sense that for every $\lambda \in (2^\X)^\omega$, we have $g(\lambda) = \omega_f(\varrho(\lambda))$, with the usual extension of $\delta$ to words over $2^\X$ from $s_0$. Note that there are many possible choices for the output function $\omega_f$. The transducer $\T$ defines a winning strategy by restricting $\omega_f$ to return only one possible setting of $\Y$.

We extract the output function $\omega_f: Q \times 2^\X \rightarrow 2^\Y$ for the game from the approximates for $Z$ assuming $\hat{Z}$ to be $Sys_{f}$, from where no matter what the environment strategy is, traces have to always get $\neg \alpha$ hold. Thus, we consider:
\begin{centering}
	$\mu Z.(\forall X. \exists Y.((X \models \neg \alpha \vee \delta(s, X \cup Y ) \in Z \cup Acc) \wedge \delta(s, X \cup Y ) \in Sys_{f} \cup Acc))$
\end{centering}
with approximates defined as:

\noindent
$\bullet~Z_{0} = \emptyset$,\\
$\bullet~Z_{i+1} = \forall X. \exists Y.((X \models \neg \alpha \vee  \delta(s, X \cup Y ) \in Z_i \cup Acc ) \wedge \delta(s, X \cup Y ) \in Sys_{f} \cup Acc)$.

Define an output function $\omega_f: Sys_{f} \times 2^\X \rightarrow 2^\Y$ as follows: for $s \in Z_{i+1} \backslash Z_i$, for all possible values $X \in 2^\X$, set $Y$ to be such that $(X \models \neg \alpha \vee  \delta(s, X \cup Y ) \in Z_i \cup Acc) \wedge \delta(s, X \cup Y ) \in Sys_{f} \cup Acc$ holds for $s \notin Acc$. Consider a deterministic finite transducer $\T$ defined in the sense that constructing $\omega_f$ as described above, the following theorem guarantees that $\T$ generates an agent winning strategy $g$.

\begin{theorem}
	Strategy $g$ with $g(\lambda) = \omega_f(\varrho(\lambda))$ is a winning strategy for the agent.
\end{theorem}

\section{Solution to Stable \LTLf Synthesis}~\label{sec:stable}
Solving stable \LTLf synthesis problem $\langle \X, \Y, \alpha, \phi \rangle$ relies on solving the stable \DFA game $\langle \G_\phi, \alpha \rangle$, where $\G_\phi$ is the corresponding \DFA of $\phi$. The following theorem guarantees the correctness of such reduction.

\begin{theorem}
	Stable \LTLf synthesis problem $\langle \X, \Y, \alpha, \phi \rangle$ is realizable iff stable \DFA game $\langle \G_\phi, \alpha \rangle$ is realizable.
\end{theorem}

\begin{proof}
	We prove the theorem in both directions.
	
	\noindent
	$\leftarrow: $ 
	Since $\langle \G_\phi, \alpha \rangle$ is realizable for the agent, the initial state $s_0$ is an agent winning state with winning strategy $g: (2^{\X})^+ \rightarrow 2^{\Y}$. Therefore, a play $\rho = (s_0, X_0\cup g(X_0)), (s_1, X_1\cup g(X_0,X_1)), \ldots$ over $\G_\phi$ from $s_0$ following $g$ is a winning play for the agent. Moreover, for every such play $\rho$ from $s$, either of the following conditions holds:
	
	\noindent
	$\bullet~\rho \nvDash FG\alpha$ such that $\rho$ is not $\alpha$-stable.\\
	$\bullet~\rho \models FG\alpha$ such that $\rho$ is $\alpha$-stable. Since $\rho$ is winning for the agent, there exists $j \geq 0$ such that $s_j \in Acc$. Therefore, $\rho^j \models \phi$ holds, where $ \rho^j = (s_0, X_0\cup g(X_0)), (s_1, X_1\cup g(X_0,X_1)), \ldots , (s_j, X_j\cup g(X_0,X_1,\ldots,X_{j}))$.
	
	Consequently, the strategy $g$ assures that for an arbitrary environment trace $\lambda = X_0,X_1,\ldots\in (2^{\X})^{\omega}$, if $\lambda$ is $\alpha$-stable, then there exists $j\geq 0$ such that $\phi$ is $true$ in finite trace $\rho^j$. Thus $\langle \X, \Y, \alpha, \phi \rangle$ is realizable.
	
	\noindent
	$\rightarrow: $ For this direction, we assume that $\langle \X, \Y, \alpha, \phi \rangle$ is realizable, then there exists a strategy $g: (2^{\X})^+ \rightarrow 2^{\Y}$ that realizes $\phi$. Thus consider an arbitrary environment trace $\lambda \in (2^\X)^\omega$, either of the following conditions holds:
	
	\noindent
	$\bullet~\lambda$ is not $\alpha$-stable, then the induced play $\rho = (s_0, X_0\cup g(X_0)), (s_1, X_1\cup g(X_0,X_1)), \ldots$ over $\G_\phi$ from $s_0$ that follows $g$ is winning for the agent by default.\\
	$\bullet~\lambda$ is $\alpha$-stable, then on the induced play $\rho = (s_0, X_0\cup g(X_0)), (s_1, X_1\cup g(X_0,X_1)), \ldots$ over $\G_\phi$ from $s_0$, there exists $j \geq 0$ such that $\phi$ is $true$ in the finite trace $\rho^j = (s_0, X_0\cup g(X_0)), (s_1, X_1\cup g(X_0,X_1)), \ldots, (s_j, X_j\cup g(X_0,X_1,\ldots,X_{j}))$, in which case $s_j \in Acc$. Therefore, $\rho$ is winning for the agent.
	
	Consequently, we conclude that stable \DFA game $\langle \G_\phi, \alpha \rangle$ is realizable for the agent. 
\end{proof}

\subsection{Stable \DFA Game Solving}
Despite the duality between fairness and stability, solving the stable \DFA game here cannot directly dualize the solution to fair \DFA game. This is because the computation here involves a \emph{Stability-Safety} game, which is not dual to the \emph{Recurrence-Safety} game in fair \DFA game solving. In order to deal with stable \DFA game, we again first consider the environment as the protagonist. We compute the set of environment winning states as follows:

\begin{centering}
	$Env_{st}= \mu Z.\nu \hat{Z}.(\exists X. \forall Y. ( (X \models \alpha  \wedge 
	\delta(s, X  \cup Y) \in \hat{Z} \backslash Acc) \vee \delta(s, X \cup Y) \in Z\backslash Acc))$,
\end{centering}

where $X$ ranges over  $2^\X$ and $Y$ over $2^\Y$. 

The following theorem assures that the nested fixpoint computation of $Env_{st}$ collects exactly all environment winning states in stable \DFA game.

\begin{theorem}
	For a stable \DFA game $\langle \G, \alpha \rangle$ and a state $s \in S$, we have $s \in Env_{st}$ iff $s$ is an environment winning state.
\end{theorem}

Correspondingly, since stable \DFA game is determined, the set of agent winning states can be computed as follows:

\begin{centering}
	$Sys_{st} = \nu Z. \mu \hat{Z}.(\forall X. \exists Y. ((X \models \neg \alpha \vee 
	\delta(s, X \cup Y) \in \hat{Z} \cup Acc) \wedge \delta(s, X \cup Y) \in Z \cup Acc))$.
\end{centering}

\begin{theorem}\label{thm:stable_winning_states}
	A stable \DFA game $\langle \G, \alpha \rangle$ has an agent winning strategy if and only if $s_0 \in Sys_{st}$.
\end{theorem}

\subsection{Strategy Extraction} Here, the agent \emph{winning strategy} $g : (2^\X)^+ \rightarrow 2^\Y$ can also be represented as a deterministic finite transducer $\T = (2^\X, 2^\Y, Q, s_0, \varrho, \omega_{st})$ in terms of the set of agent winning states such that $Q = Sys_{st}$.

We extract the output function $\omega_{st}: Q \times 2^\X \rightarrow 2^\Y$ for the game from the approximates for $Z$ assuming $\hat{Z}$ to be $Sys_{st}$, from where no matter what the environment strategy is, traces cannot always get $\alpha$ hold. Thus, we consider the fixpoint computation as follows:

\begin{centering}
	$\nu Z.(\forall X. \exists Y. ((X \models \neg \alpha \vee \delta(s, X \cup Y) \in Sys_{st} \cup Acc) \wedge \delta(s, X \cup Y) \in Z \cup Acc))$.
\end{centering}

Define an output function $\omega_{st}: Sys_{st} \times 2^\X \rightarrow 2^\Y$ s.t. for $s \in Z_{i+1} \cap Z_i$, for all possible values $X \in 2^\X$, set $Y$ to be s.t. $(X \models \neg \alpha \vee \delta(s, X \cup Y) \in Sys_{st} \cup Acc) \wedge \delta(s, X \cup Y) \in Z_i \cup Acc$ holds for $s \notin Acc$. The following theorem guarantees that $\T$ generates an agent winning strategy $g$.

\begin{theorem}
	Strategy $g$ with $g(\lambda) = \omega_{st}(\varrho(\lambda))$ is a winning strategy for the agent.
\end{theorem}

\section{Evaluation}
We observe that a straightforward approach to \LTLf synthesis under assumptions can be obtained by a reduction to standard \LTL synthesis, which allows us to utilize tools for \LTL synthesis to solve the fair~(or stable) \LTLf synthesis problem. In this section, we first revisit the reduction to standard \LTL synthesis, and then show an experimental comparison with the approach proposed earlier in this paper. 

\noindent\textbf{Reduction to \LTL Synthesis.}
The insight of reducing \LTLf synthesis under assumptions to \LTL synthesis comes from the reduction in~\cite{ZhuTLPV17} for general \LTLf synthesis, and in~\cite{CamachoBM18} for constraint \LTLf synthesis, where the constraint describes the desired environment behaviors, under which the goal is to satisfy the given \LTLf specification. Both reductions adopt the translation rules in~\cite{DV13} to polynomially transform an \LTLf formula $\phi$ over $\X \cup \Y$ into an \LTL formula $\psi$ over $\X \cup \Y \cup \{alive\}$, retaining the satisfiability equivalence, where proposition $alive$ indicates the last instance of the finite trace. Such translation bridges the gap between \LTLf over finite traces and \LTL over infinite traces. Based on the translation from \LTLf to \LTL, we then reduce fair~(resp., stable) \LTLf synthesis problem $\langle \X, \Y, \alpha, \phi \rangle$ to \LTL synthesis problem $\langle \X, \Y \cup \{alive\}, GF \alpha \rightarrow \psi \rangle$~(resp., $\langle \X, \Y \cup \{alive\}, FG \alpha \rightarrow \psi \rangle$). 

\noindent\textbf{Implementation.}
Based on the \LTLf synthesis tool~\emph{Syft}~\footnote{https://github.com/saffiepig/Syft}, we implemented our fixpoint-based techniques for solving fair \LTLf synthesis and stable \LTLf synthesis in two tools called \emph{FSyft} and \emph{StSyft}, respectively~(name after~\emph{Syft}). Both frameworks consist of two steps: the symbolic \DFA construction and the respective \DFA game solving. In the first step, we based on the code of~\emph{Syft}, to construct the symbolic \DFA represented in Binary Decision Diagrams~(BDDs). The implementation of the nested fixpoint computation for solving \DFA games over such symbolic \DFA, borrows techniques from~\cite{ZTLPV17} for greatest fixpoint computation and from~\emph{Syft} for least fixpoint computation. The construction of the transducer for generating the winning strategy utilizes the boolean-synthesis procedure introduced in~\cite{DrLu.cav16} for realizable formulas. The implementation makes use of the BDD library CUDD-3.0.0~\cite{cudd}. In order to evaluate the performance of \emph{FSyft} and \emph{StSyft}, we compared it against the solution of reducing to standard \LTL synthesis shown above. For such comparison, we employed the \LTLf-to-\LTL translator implemented in SPOT~\cite{spot} and chose Strix~\cite{MeyerSL18}, 
the winner of the synthesis competition SYNTCOMP 2019~\footnote{http://www.syntcomp.org/syntcomp-2019-results/} over \LTL synthesis track, as the baseline. 

\subsection{Experimental Methodology}

\subsubsection{Benchmarks.} We collected 1200 formulas consisting of two classes of benchmarks: 1000 randomly conjuncted \LTLf formulas over 100 basic cases, generated in the style described in~\cite{ZhuTLPV17}, the length of which, indicating the number of conjuncts, ranges form 1 to 5. The assumption~(either fairness or stability) is assigned by randomly selecting one variable from all environment variables; 200 \LTLf synthesis benchmarks with assumptions generated from a scalable counter game, described as follows:

\noindent
$\bullet$~There is an $n$-bit binary counter. At each round, the environment chooses whether to increment the counter or not. The agent can choose to grant the request or ignore it.\\
$\bullet$~The goal is to get the counter having all bits set to $1$, so the counter reaches the maximal value.\\
$\bullet$~The fairness assumption is to have the environment infinitely request the counter to be incremented.\\
$\bullet$~The stability assumption is to have the environment eventually keep requesting the counter to be incremented.


We reduce solving the counter game above to solving \LTLf synthesis with assumptions. First, we have $n$ agent variables $\{b_{n-1}, b_{n-2}, \ldots, b_0 \}$ denoting the value of $n$ counter bits. We also introduce another $n+1$ agent variables $\{c_{n}, c_{n-1}, \ldots, c_0\}$ representing the carry bits. In addition, we have an environment variable $add$ representing the environment making an increment request or not, and $c_0$ as $true$ is considered as the agent granting the request. We then formulate the counter game into \LTLf formula as follows: 
$Init = ((\neg c_0) \wedge \ldots \wedge (\neg c_{n-1}) \wedge (\neg b_0) \wedge  \ldots (\neg b_{n-1}))), \\
Goal = F(b_0 \wedge \ldots \wedge b_{n-1}), \\
B = G((\neg add)\rightarrow X_w(\neg c_0)), \\
B_i = \begin{cases}
 (((\neg c_i) \wedge (\neg b_i)) \rightarrow X_w((\neg b_i) \wedge (\neg c_{i+1})))\\
 (((\neg c_i) \wedge b_i) \rightarrow X_w(b_i \wedge (\neg c_{i+1})))\\
 (((c_i \wedge \neg b_i) \rightarrow X_w(b_i \wedge (\neg c_{i+1}))) \\
 (((c_i \wedge b_i) \rightarrow X_w((\neg b_i) \wedge c_{i+1}))).  \end{cases}$

The \LTLf formula $\phi$ is then $(Init \wedge B \wedge \bigwedge_{0 \leq i \leq n} G(B_i)) \wedge Goal$, and the constraint $\alpha$ is $add$. Obviously, such counter game only returns realizable cases, since a winning strategy for the agent is to grant all increment requests.

In order to get unrealizable cases, we can make some modifications on the counter game above. One possibility is to have the counter increment by 2 if the agent chooses to grant the request sent by the environment. Such modification leads to no winning strategy for the agent, since the maximal counter value of having each bit as $1$ is odd. However, incrementing by 2 at each time will never reach an odd value. Therefore, for bit $B_i$ such that $i > 0$, we keep the same formulation. While for bit $B_0$, we change as follows:
$B_0 = \begin{cases}
 (((\neg c_0) \wedge (\neg b_0)) \rightarrow X_w((\neg b_0) \wedge (\neg c_1)))\\
 (((\neg c_0) \wedge b_0) \rightarrow X_w(b_0 \wedge (\neg c_1)))\\
 (((c_0 \wedge \neg b_0) \rightarrow X_w(\neg b_0 \wedge (c_1))) \\
 (((c_0 \wedge b_0) \rightarrow X_w((b_0) \wedge c_1))).  \end{cases}$

Therefore, we have 200 counter game benchmarks in total, with the number of counter bits $n$ ranging from 1 to 100, and both realizable and unrealizable cases for each $n$.

\noindent\textbf{Experiment Setup.} All tests were ran on a computer cluster. Each test took an exclusive access to a node with Intel(R) Xeon(R) CPU E5-2650 v2 processors running at 2.60GHz. Time out
was set to 1000 seconds.

\noindent\textbf{Correctness.} Our implementation was verified by comparing
the results returned by \emph{FSyft} and \emph{StSyft} with those from Strix. No inconsistency encountered for the solved cases.
\subsection{Experimental Results.}
\begin{figure}[t]
	\centering
	\includegraphics[width=.85\columnwidth]{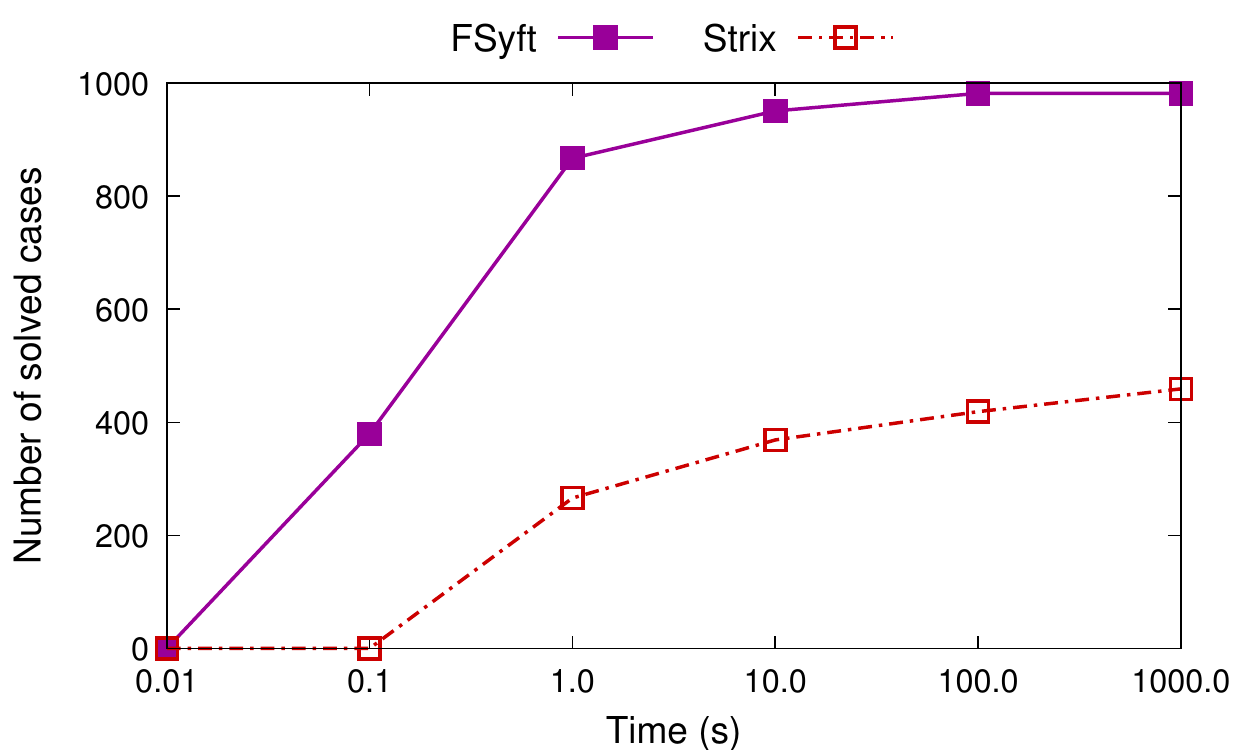}
	\caption{Fair \LTLf synthesis. Comparison of the number of solved cases with limited time between \emph{FSyft} and Strix over random conjunction benchmarks.}
	\label{fig:fsyntime}
\end{figure}

\begin{figure}[t]
	\centering
	\includegraphics[width=.85\columnwidth]{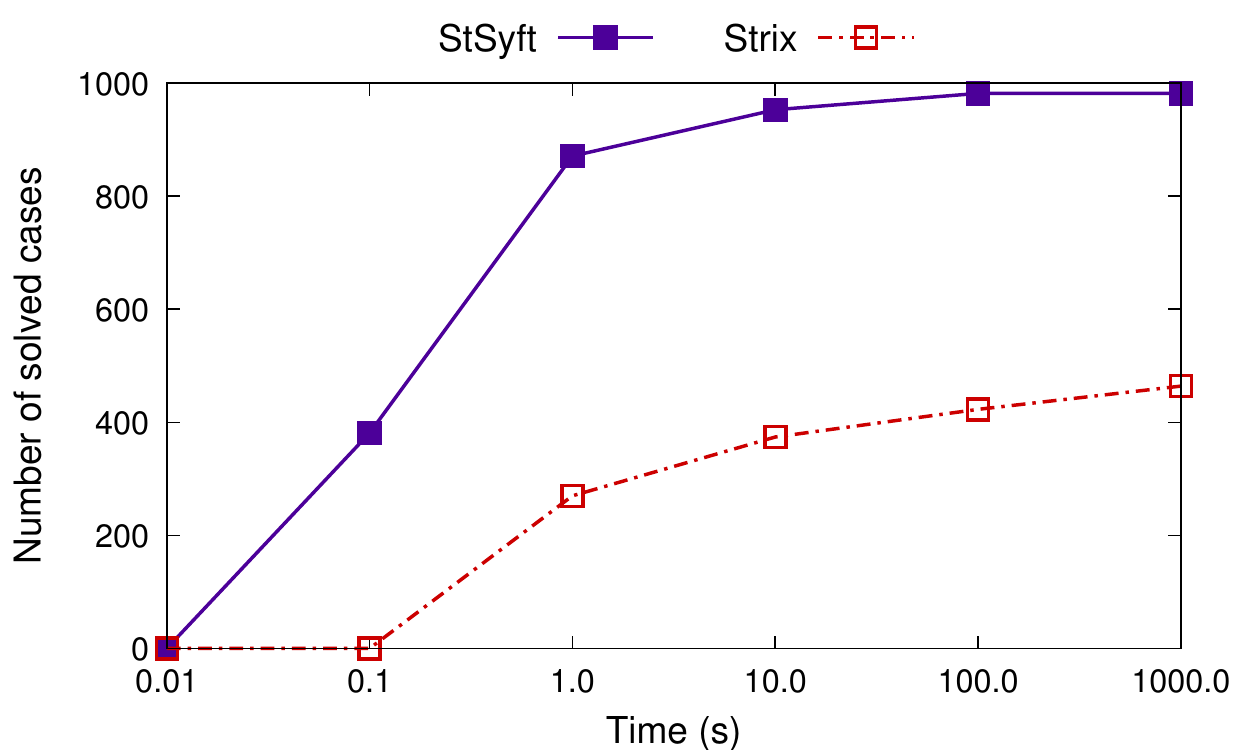}
	\caption{Stable \LTLf synthesis. Comparison of the number of solved cases with limited time between \emph{StSyft} and Strix over random conjunction benchmarks.}
	\label{fig:stsyntime}
\end{figure}
We evaluated the efficiency of \emph{FSyft} and \emph{StSyft} in terms of the number of solved cases and total time cost. We compared these two tools against Strix by performing an end-to-end comparison experiment. Therefore, both of the \DFA construction time and the fixpoint computation time were counted for \emph{FSyft} and \emph{StSyft}. For Strix, we counted the running time from feeding the corresponding \LTL formula to Strix to receiving the result. Both comparison on two classes of benchmarks show the advantage of the fixpoint-based technique proposed in this paper as an effective method for both of fair \LTLf synthesis and stable \LTLf synthesis~\footnote{We recommend viewing the figures online for a better vision.}.

\noindent\textbf{Randomly Conjuncted Benchmarks.}
Figure~\ref{fig:fsyntime} and Figure~\ref{fig:stsyntime} show the number of solved cases as the given time increases on fair \LTLf synthesis and stable \LTLf synthesis, respectively. As shown in the figures, both of \emph{FSyft} and \emph{StSyft} are able to handle almost all cases~(1000 in total for each), while Strix only solves a small fraction of the cases that \emph{FSyft} and \emph{StSyft} can solve. Moreover, as presented there, half of the cases that can be solved by \emph{FSyft} and \emph{StSyft}, around 400, are finished in less than 0.1 second, while Strix is unable to solve any cases given such time limit.

\noindent\textbf{Counter Game.}
Figure~\ref{fig:fcounter} and Figure~\ref{fig:stcounter} show the running time of all tools on the counter game benchmarks. Since all of them got failed on cases with counter bits $n>10$, here we only show realizable/unrealizale cases with counter bits $n \leq 10$, so we have 20 cases for each synthesis problem. The x-labels \emph{c-rea/unrea-n} indicate the realizability and the number of counter bits of each case. Both of \emph{FSyft} and \emph{StSyft} are able to deal with cases with $n\leq10$, while Strix only solves cases with $n$ up to 7, either stable \LTLf synthesis or fair \LTLf synthesis. For those common solved cases, both of \emph{FSyft} and \emph{StSyft} take much less time than Strix.

\begin{figure}[t]
	\centering
	\includegraphics[width=.9\columnwidth]{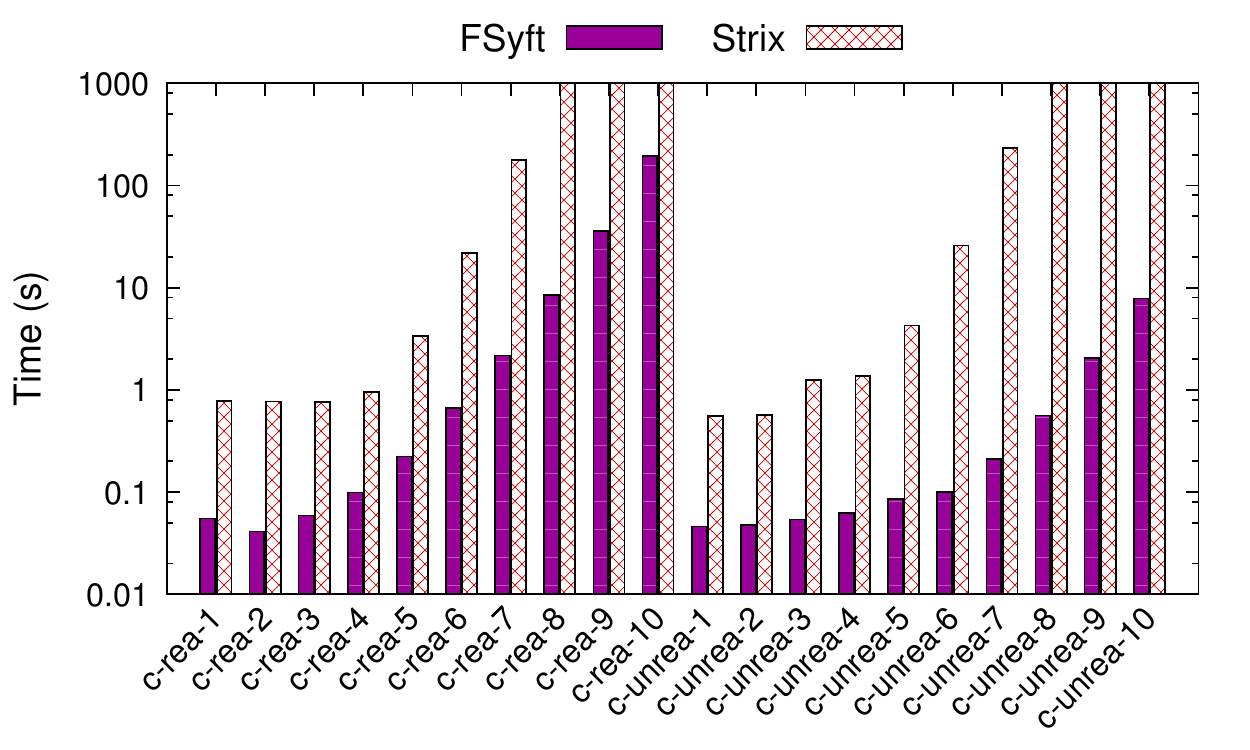}
	\caption{Fair \LTLf synthesis. Comparison of running time between \emph{FSyft} and Strix, in log scale. Bars of the maximum height indicate cases timed out.}
	\label{fig:fcounter}
\end{figure}

\begin{figure}[t]
	\centering
	\includegraphics[width=.9\columnwidth]{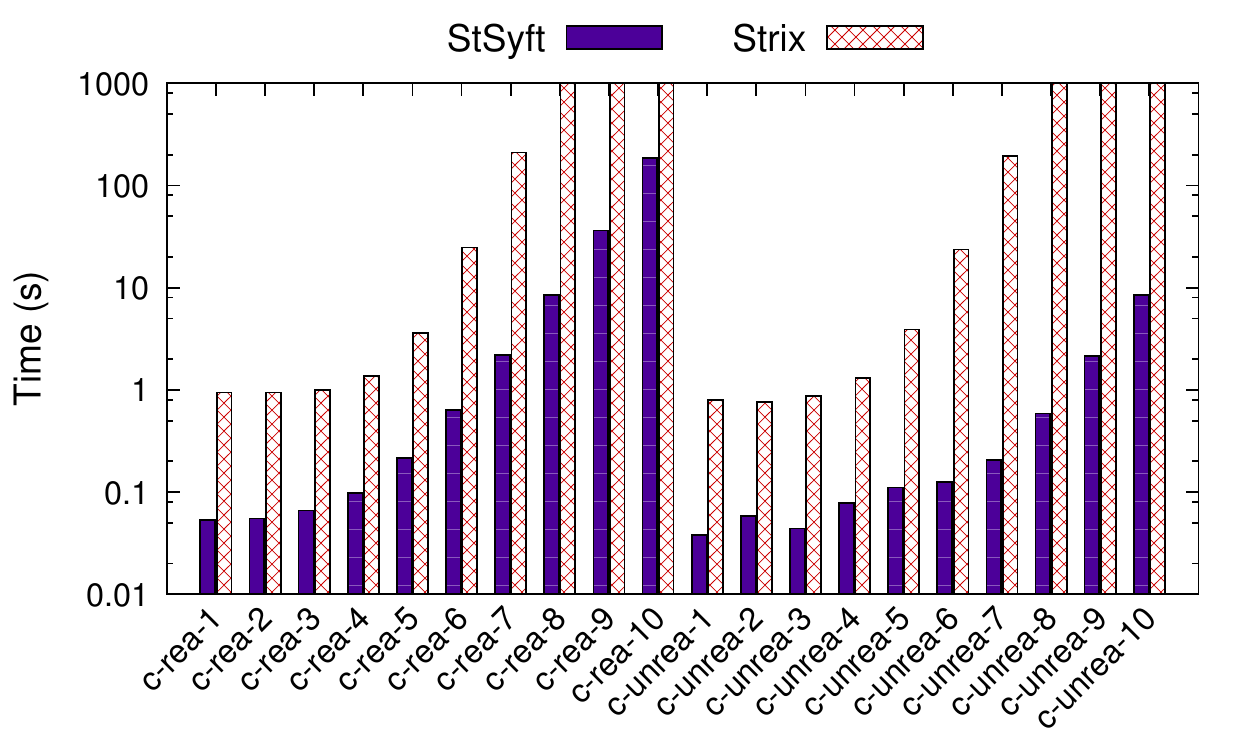}
	\caption{Stable \LTLf synthesis. Comparison of running time between \emph{StSyft} and Strix, in log scale. Bars of the maximum height indicate cases timed out.}
	\label{fig:stcounter}
\end{figure}

\section{Conclusions}
In this paper we presented a fixpoint-based technique for \LTLf synthesis with assumptions for basic forms of fairness and stability, which is quite effective, as our experiment shows. Our technique can be summarized as follows: use the \DFA for the \LTLf formula as the arena to play a game for the environment whose winning condition is to avoid reaching the accepting states while making the assumption true. Note that for a general \LTL assumption~(see~\cite{ADMRicaps19}), we can transform such an assumption into a parity automaton, take the Cartesian product with the \DFA and play the parity/reachability game over the resulting arena. Comparing this possible approach to the reduction to \LTL synthesis is a subject for future work.

\noindent\textbf{Acknowledgments.}
Work supported in part by European Research Council under the European Union’s Horizon 2020 Programme through the ERC Advanced Grant WhiteMec (No.\ 834228), NSF grants IIS-1527668, CCF-1704883, and IIS-1830549, NSFC Projects No.\ 61572197, No.\ 61632005 and No.\ 61532019. 

\bibliographystyle{aaai}
\bibliography{AAAI-ZhuS.3174}

\appendix
\section{Appendix}
Due to the lack of space, we move some proofs and the details of the reduction from fair \LTLf synthesis and stable \LTLf synthesis to standard \LTL synthesis in this appendix.

For better readability, we redefine two-players \DFA games here. Two-player games on \DFA are games consisting of two players, the \emph{environment} and the \emph{agent}. $\X$ and $\Y$ are disjoint sets of environment Boolean variables and agent Boolean variables, respectively.
The \emph{specification} of the game arena is given by a \DFA $\G$ = $(2^{\X\cup \Y}, S, s_0, \delta, Acc)$, where

\begin{itemize}\itemsep=0pt
	\item $2^{\X\cup \Y}$ is the alphabet;
	\item $S$ is a set of states; 
	\item $s_0\in S$ is an initial state;
	\item $\delta: S\times 2^{\X\cup \Y}\rightarrow S$ is a transition function;
	\item $Acc\subseteq S$ is a set of accepting states.
\end{itemize}

Here, we consider two specific two-player games, fair \DFA game and stable \DFA game, both of which are described as $\langle \G, \alpha \rangle$, where $\G$ is the game arena and $\alpha$ is the environment constraint, which is a Boolean formula over $\X$.
\subsection{Fair \LTLf Synthesis}
Due to the determinacy of fair \DFA game, the set of agent winning states $Sys_{f}$ can be computed by negating $Env_f$:

\begin{centering}
	$Sys_{f} = \mu Z.\nu \hat{Z}.(\forall X. \exists Y. ((X \models \neg \alpha \vee
	\delta(s, X \cup Y) \in Z \cup Acc) \wedge \delta(s, X \cup Y) \in \hat{Z} \cup Acc))$.
\end{centering}

The following theorem guarantees the correctness of the set of agent winning states computation $Sys_{f}$.
\addtocounter{theorem}{+2}
\begin{theorem}\label{thm:winning_states}
	A fair \DFA game $\langle \G, \alpha \rangle$ has an agent winning strategy if and only if $s_0 \in Sys_{f}$.
\end{theorem}
\begin{proof}
	Since $Sys_{f}$ is the dual formula of $Env_{f}$, for a state $s \in S$, we have $s \in Sys_{f}$ if and only if $s \notin Env_{f}$ such that $s$ is not a winning state for the environment, in which case $s$ is an agent winning state with winning strategy $g : (2^\X)^+ \rightarrow 2^\Y$. Therefore, for a state $s \in S$, we have $s \in Sys_{f}$ if and only if $s$ is an agent winning state. Moreover, fair \DFA game is realizable if and only if the initial state $s_0$ is an agent winning state. Consequently, we conclude that fair \DFA game $\langle \G, \alpha \rangle$ is realizable with agent winning strategy $g$ if and only if $s_0 \in Sys_{f}$. 
\end{proof}

Define an output function $\omega_f: Sys_{f} \times 2^\X \rightarrow 2^\Y$ as follows: for $s \in Z_{i+1} \backslash Z_i$, for all possible values $X \in 2^\X$, set $Y$ to be such that $(X \models \neg \alpha \vee  \delta(s, X \cup Y ) \in Z_i \cup Acc) \wedge \delta(s, X \cup Y ) \in Sys_{f} \cup Acc$ holds for $s \notin Acc$. Consider a deterministic finite transducer $\T$ defined in the sense that constructing $\omega_f$ so as described above, the following theorem guarantees that $\T$ generates an agent winning strategy $g$. The following theorem guarantees that deterministic finite transducer $\T$ is able to generate a winning strategy $g$ for the agent.
\begin{theorem}
	Strategy $g$ with $g(\lambda) = \omega_f(\varrho(\lambda))$ is a winning strategy for the agent.
\end{theorem}
\begin{proof}
	Consider an arbitrary environment trace $\lambda = X_0,X_1,\ldots\in (2^{\X})^{\omega}$, the corresponding play over $\G$ that follows $g$ is $\rho = (s_0, X_0\cup g(X_0)), (s_1, X_1\cup g(X_0,X_1)), \ldots$. We now prove that $\rho$ is a winning play for the agent. For every state $s$ along the play $\rho$, the construction of $\omega_f$ ensures that, no matter how the environment sets $X$, $\omega_f$ returns $Y$ such that $(X \models \neg \alpha \vee  \delta(s, X \cup Y ) \in Z_i \cup Acc ) \wedge \delta(s, X \cup Y ) \in Sys_{f} \cup Acc$ holds. Thus we either have $\neg \alpha$ holds, or $\rho$ visits $Z_i \cup Acc$. At the same time, $\rho$ keeps in $Sys_{f} \cup Acc$. The first possibility keeps the stability condition and the latter one retains the reachability condition, by inductive hypothesis, both of them give $\rho$ a winning play. Therefore, $g$ is a winning strategy for the agent.
\end{proof}
\subsection{Stable \LTLf Synthesis}
In stable \DFA game $\langle \G, \alpha \rangle$, we compute the set of environment winning states as follows:

\begin{centering}
	$Env_{st}= \mu Z.\nu \hat{Z}.(\exists X. \forall Y. ( (X \models \alpha  \wedge 
	\delta(s, X  \cup Y) \in \hat{Z} \backslash Acc) \vee \delta(s, X \cup Y) \in Z\backslash Acc))$,
\end{centering}

where $X$ ranges over  $2^\X$ and $Y$ over $2^\Y$. 

The fixpoint stages for $Z$~(note $Z_{i} \subseteq Z_{i+1}$, for $i \geq 0$, by monotonicity) are:
\begin{itemize}\itemsep=0pt
	\item 
	$Z_0 = \emptyset$,
	\item 
	$Z_{i+1} = \nu \hat{Z}.(\exists X. \forall Y. ((X \models \alpha \wedge 
	\delta(s, X \cup Y) \in \hat{Z} \backslash Acc) \vee \delta(s, X \cup Y) \in Z_i \backslash Acc))$.
\end{itemize}
Eventually, $Env_{st} = Z_k$ for some $k$ such that $Z_{k+1} = Z_{k}$.

The fixpoint stages for $\hat{Z}$ with respect to $Z_i$~(note $\hat{Z}_{j+1} \subseteq \hat{Z}_{j}$, for $j \geq 0$, by monotonicity) are:

\begin{itemize}\itemsep=0pt
	\item 
	$\hat{Z}_{i,0} = S$,
	\item
	$\hat{Z}_{i,j+1} = \exists X. \forall Y. ((X \models \alpha \wedge 
	\delta(s, X \cup Y) \in \hat{Z}_{i,j} \backslash Acc) \vee \delta(s, X \cup Y) \in Z_i \backslash Acc)$.
\end{itemize}

Finally, $\hat{Z}_i = \hat{Z}_{i,k}$ for some $k$  such that $\hat{Z}_{i,k+1} = \hat{Z}_{i,k}$.
The following theorem assures that the nested fixpoint computation of $Env_{st}$ collects exactly all environment winning states in stable \DFA game.
\addtocounter{theorem}{+1}
\begin{theorem}
	For a stable \DFA game $\langle \G, \alpha \rangle$ and a state $s \in S$, we have $s \in Env_{st}$ iff $s$ is an environment winning state.
\end{theorem}

\begin{proof}
	We prove the theorem in both directions.
	
	$\leftarrow:$ We proceed the proof by showing the contropositive. A state $s \notin Env_{st}$ indicates that $s$ cannot be added to $Z_{i+1}$ at stage $i+1$ for all $i \geq 0$. Then $s \notin \nu \hat{Z}.(\exists X. \forall Y. ((X \models \alpha \wedge 
	\delta(s, X \cup Y) \in \hat{Z} \backslash Acc) \vee \delta(s, X \cup Y) \in Z \backslash Acc))$. That is, no matter what the~(environment)~strategy $h$ is, traces from $s$ satisfy \textbf{neither} of the following conditions: 
	\begin{itemize}
		\item $\alpha$ holds and the trace gets trapped in $\hat{Z}$ without visiting any accepting states such that $X \models \alpha \wedge 
		\delta(s, X \cup Y) \in \hat{Z} \backslash Acc$ holds, in which case, $s$ is a new environment winning state;
		\item one already defined environment winning state gets visited such that $\delta(s, X \cup Y) \in Z \backslash Acc$ holds, in which case, $s$ is a new environment winning state.
	\end{itemize}
	Therefore, $s$ is not an environment winning state. So if $s$ is an environment winning state then $s \in Env_{st}$ holds. 
	
	$\rightarrow:$ If a state $s \in Env_{st}$, then $s \in \nu \hat{Z}.(\exists X. \forall Y. ((X \models \alpha \wedge 
	\delta(s, X \cup Y) \in \hat{Z} \backslash Acc) \vee \delta(s, X \cup Y) \in Z \backslash Acc))$. That is, no matter what the~(system)~strategy $g$ is, traces from $s$ satisfy either of the following conditions:
	\begin{itemize}
		\item $\alpha$ holds and the trace gets trapped in $\hat{Z}$ without visiting any accepting states such that $X \models \alpha \wedge 
		\delta(s, X \cup Y) \in \hat{Z} \backslash Acc$ holds, in which case, $s$ is a new environment winning state;
		\item one already defined environment winning state gets visited such that $\delta(s, X \cup Y) \in Z \backslash Acc$ holds, in which case, $s$ is a new environment winning state.
	\end{itemize}
	Thus $s$ is a winning state for the environment. 
\end{proof}

In stable \DFA game $\langle \G, \alpha \rangle$, the set of agent winning states can be computed as follows:

\begin{centering}
	$Sys_{st} = \nu Z. \mu \hat{Z}.(\forall X. \exists Y. ((X \models \neg \alpha \vee 
	\delta(s, X \cup Y) \in \hat{Z} \cup Acc) \wedge \delta(s, X \cup Y) \in Z \cup Acc))$
\end{centering}
\begin{theorem}\label{thm:winning_states}
	A stable \DFA game $\langle \G, \alpha \rangle$ has an agent winning strategy if and only if $s_0 \in Sys_{st}$.
\end{theorem}
\begin{proof}
	Since $Sys_{st}$ is the dual formula of $Env_{st}$, for a state $s \in S$, we have $s \in Sys_{st}$ if and only if $s \notin Env_{st}$ such that $s$ is not a winning state for the environment, in which case $s$ is an agent winning state with winning strategy $g : (2^\X)^+ \rightarrow 2^\Y$. Therefore, for a state $s \in S$, we have $s \in Sys_{st}$ if and only if $s$ is an agent winning state. Moreover, stable \DFA game is realizable if and only if the initial state $s_0$ is an agent winning state. Consequently, we conclude that stable \DFA game $\langle \G, \alpha \rangle$ is realizable with agent winning strategy $g$ if and only if $s_0 \in Sys_{st}$. 
\end{proof}

We extract the output function $\omega_{st}: Q \times 2^\X \rightarrow 2^\Y$ for the game from the approximates for $Z$ assuming $\hat{Z}$ to be $Sys_{st}$, from where no matter what the environment strategy is, traces cannot always get $\alpha$ hold. Thus, we consider the fixpoint computation as follows:

\begin{centering}
	$\nu Z.(\forall X. \exists Y. ((X \models \neg \alpha \vee \delta(s, X \cup Y) \in Sys_{st} \cup Acc) \wedge \delta(s, X \cup Y) \in Z \cup Acc))$
\end{centering}

with approximates defined as:
\begin{itemize}\itemsep=0pt
	\item 
	$Z_{0} = S$,
	\item
	$Z_{i+1} = \forall X. \exists Y. ((X \models \neg \alpha \vee \delta(s, X \cup Y) \in Sys_{st} \cup Acc) \wedge \delta(s, X \cup Y) \in Z_i \cup Acc)$.
\end{itemize}

Define an output function $\omega_{st}: Sys_{st} \times 2^\X \rightarrow 2^\Y$ as follows: for $s \in Z_{i+1} \cap Z_i$, for all possible values $X \in 2^\X$, set $Y$ to be such that $(X \models \neg \alpha \vee \delta(s, X \cup Y) \in Sys_{st} \cup Acc) \wedge \delta(s, X \cup Y) \in Z_i \cup Acc$ holds for $s \notin Acc$. Consider a deterministic finite transducer $\T$ defined in the sense that constructing $\omega_{st}$ so as described above, the following theorem guarantees that $\T$ generates an agent winning strategy $g$.

\begin{theorem}
	Strategy $g$ with $g(\lambda) = \omega_{st}(\varrho(\lambda))$ is a winning strategy for the agent.
\end{theorem}
\begin{proof}
	Consider an arbitrary environment trace $\lambda = X_0,X_1,\ldots\in (2^{\X})^{\omega}$, the corresponding play over $\G$ that follows $g$ is $\rho = (s_0, X_0\cup g(X_0)), (s_1, X_1\cup g(X_0,X_1)), \ldots$. We now prove that $\rho$ is a winning play for the agent. For every state $s$ along the play $\rho$, the construction of $\omega_{st}$ ensures that, no matter how the environment sets $X$, $\omega_{st}$ returns $Y$ such that $s \in (X \models \neg \alpha \vee \delta(s, X \cup Y) \in Sys_{st} \cup Acc) \wedge \delta(s, X \cup Y) \in Z_i \cup Acc$ holds. Thus we either have $\neg \alpha$ holds, or $\rho$ stays in $Sys_{st} \cup Acc$. At the same time, $\rho$ visits $Z_i \cup Acc$. The first possibility keeps the recurrence condition and the latter one remains the reachability condition, by inductive hypothesis, both of them give $\rho$ a winning play. Therefore, $g$ is a winning strategy for the agent.
\end{proof}
\subsection{Reduction to \LTL Synthesis}\label{sec:LTL}
In addition to the fixpoint-based automata theoretical solution, an alternative approach to fair \LTLf synthesis and stable \LTLf synthesis can be obtained by a reduction to standard \LTL synthesis. 

\addtocounter{definition}{+6}
\begin{definition}[\LTL Synthesis]
	Let $\psi$ be an \LTL formula over an alphabet $\mathcal{P}$ and $\X, \Y$ be two disjoint atom sets such that $\X \cup \Y = \mathcal{P}$. $\psi$ is \emph{realizable} with respect to $\langle\X, \Y\rangle$ if there exists a strategy $f: (2^{\X})^+ \rightarrow 2^{\Y}$, such that for an arbitrary infinite sequence $\lambda = X_0,X_1,\ldots\in (2^{\X})^{\omega}$, $\psi$ is true in the infinite trace $\rho= (X_0\cup f(X_0)), (X_1\cup f(X_0,X_1)), (X_2\cup f(X_0,X_1,X_2))\ldots$. The \emph{synthesis} procedure is to compute such a strategy if $\psi$ is \emph{realizable}.
\end{definition}

Reducing fair or stable \LTLf synthesis to \LTL synthesis allows tools for general \LTL synthesis to be used in solving fair \LTLf synthesis and stable \LTLf synthesis. The reduction adopts the translation rules in~\cite{DV13} to polynomially transform an \LTLf formula $\phi$ over propositions $\X \cup \Y$ to \LTL formula $\psi$ over $\X \cup \Y \cup \{alive\}$ by introducing a new variable $alive$. As such, we have $\phi$ is satisfiable if and only if $\psi$ is satisfiable. The translation requires a function $t$ that reads an \LTLf formula and returns an \LTL formula, which is defined as follows:
\begin{itemize}
	\item 
	$t(a) = a$
	\item 
	$t(\neg \phi_1) = \neg t(\phi_1)$
	\item 
	$t(\phi_1 \wedge \phi_2) = t(\phi_1) \wedge t(\phi_2)$
	\item 
	$t(X \phi) = X(alive \wedge t(\phi))$
	\item 
	$t(\phi_1 U \phi_2) = t(\phi_1) U (alive \wedge t(\phi_2))$ 
\end{itemize}
Finally, $\psi = t(\phi) \wedge alive \wedge (alive~U~(G\neg alive))$. Since the proof of the satifiability equivalence between $\phi$ and $\psi$ is implicit in~\cite{DV13}, we show here the following lemma to ensure such relation. 
To relate a finite trace $\tau$ satisfying $\phi$ to an infinite trace $\tau'$ satisfying $\psi$, we introduce here a so called \emph{extension equivalence} as follows. $\tau$ is extension equivalent to $\tau'$, denoted as $\tau \approx \tau'$ if:
\begin{itemize}
	\item 
	$\tau'[i] = \tau[i] \wedge alive$, for $0 \leq i \leq e$, where $e$ indicates the last point of finite trace $\tau$ such that $e = |\tau| - 1$;
	\item 
	$\tau'[i] = \neg alive$, for $i > e$.
\end{itemize}

\begin{lemma}~\label{lem:sat_equ}
	Let $\phi$ be an \LTLf formula, $\psi$ be the corresponding translated \LTL formula, and $\tau$ be a finite trace, $\tau'$ be an infinite trace with $\tau \approx \tau'$. Then $\tau \models \phi$ iff $\tau' \models \psi$ is true.
\end{lemma}
\begin{proof}
	Since $|\tau| > 0$, it is straightforward to show that $\tau' \models alive \wedge (alive~U~(G\neg alive))$ since $\tau \approx \tau'$ . Now we prove the lemma by a constructive induction of $\phi$.
	\begin{itemize}
		\item 
		If $\phi = a$, then $\psi = a \wedge alive \wedge (alive~U~(G\neg alive))$, $\tau \models a$ such that $\tau[0] \models a$, in which case $\tau'[0] \models a$ such that $\tau' \models a$. Therefore, $\tau' \models \psi$.
		\item 
		If $\phi = \neg \phi_1$, then $\psi = \neg t(\phi_1) \wedge alive \wedge (alive~U~(G\neg alive))$. $\tau \models \phi$ such that $\tau \nvDash \phi_1$ holds. By induction hypothesis, $\tau' \nvDash t(\phi_1)$ such that $\tau' \models \neg t(\phi_1)$. Therefore, $\tau' \models \psi$.
		\item 
		If $\phi = \phi_1 \wedge \phi_2$, then $\psi = t(\phi_1) \wedge t(\phi_2) \wedge alive \wedge (alive~U~(G\neg alive))$. $\tau \models \phi_1 \wedge \phi_2$ such that $\tau \models \phi_1$ and $\tau \models \phi_2$ hold. By induction hypothesis, $\tau' \models t(\phi_1)$ and $\tau' \models t(\phi_2)$ hold such that $\tau' \models t(\phi_1) \wedge t(\phi_2)$. Therefore, $\tau' \models \psi$.
		\item 
		If $\phi = X \phi_1$, then $\psi = X(alive \wedge t(\phi_1)) \wedge alive \wedge (alive~U~(G\neg alive))$. $\tau \models \phi$ such that $\tau_1 \models \phi_1$ holds. By induction hypothesis, $\tau_1' \models alive \wedge t(\phi_1)$ such that $\tau' \models X(alive \wedge t(\phi))$. Therefore, $\tau' \models \psi$.
		\item 
		If $\phi = \phi_1 U \phi_2$, then $\psi = t(\phi_1) U (alive \wedge t(\phi_2)) \wedge alive \wedge (alive~U~(G\neg alive))$. $\tau \models \phi$ such that there exists $j \geq 0$ such that $\tau_j \models \phi_2$, and for $0 \leq i < j$, we have $\tau_i \models \phi_1$. By induction hypothesis, $\tau_j' \models alive \wedge t(\phi_2)$ and $\tau_i' \models t(\phi_1)$ hold for $0 \leq i < j$, thus $\tau' \models t(\phi_1) U (alive \wedge t(\phi_2))$. Therefore, $\tau' \models \psi$. 
	\end{itemize}
\end{proof}

To show the reduction from fair \LTLf synthesis and stable \LTLf synthesis to \LTL synthesis, we start by assigning the environment and agent variables. Intuitively, $alive$ is a signal whose failure indicates the end of the finite trace. Therefore, $alive$ is assigned as an agent variable such that the agent can keep setting $alive$ as $true$ until $alive$ is set to $false$ when $\phi$ is satisfied. The environment constraint $\alpha$ over environment variables in both of fair \LTLf synthesis and stable \LTLf synthesis is a condition for the satisfaction of the desired goal $\phi$. Since being realizable requires $\phi$ to be satisfied under the condition such that $\alpha$ holds infinitely often for fair \LTLf synthesis and eventually holds forever for stable \LTLf synthesis, we obtain the \LTL goal $GF\alpha \rightarrow \psi$ and $FG\alpha \rightarrow \psi$, respectively. In both cases, $\psi$ is the corresponding translated \LTL formula of $\phi$.
Thus solving the problem $\langle \X, \Y, \alpha, \phi \rangle $ is reduced to solving the \LTL synthesis problem $\langle \X, \Y \cup \{alive\}, GF \alpha \rightarrow \psi \rangle$ for fair \LTLf synthesis, and to $\langle \X, \Y \cup \{alive\}, FG \alpha \rightarrow \psi \rangle$ for stable \LTLf synthesis. The following theorems guarantee the correctness of this reduction respectively.

\begin{theorem}
Let $\phi$ be an \LTLf formula, $\psi$ be the corresponding translated \LTL formula, then fair \LTLf synthesis problem $\langle \X, \Y, \alpha, \phi \rangle$ is realizable if and only if \LTL formula $GF \alpha \rightarrow \psi$ is realizable with respect to $\langle \X, \Y \cup \{alive\} \rangle$.
\end{theorem}

\begin{proof}
 We prove the two directions separately.
 \begin{itemize}
     \item 
     $\leftarrow: $ Since $GF \alpha \rightarrow \psi$ is realizable with respect to $\langle \X, \Y \cup \{alive\} \rangle$, there exists a winning strategy $g': (2^\X)^+ \rightarrow 2^{\Y \cup \{alive\}}$ such that every trace $\rho'$ that follows $g'$ gets $GF \alpha \rightarrow \psi$ hold, therefore, enabling either of the following situations:
     \begin{itemize}
         \item 
         $GF\alpha$ is true such that the environment behaves such as having $\alpha$ hold infinitely often, and $\psi$ holds. Therefore, $\rho' \models alive \wedge (alive~U~(G\neg alive)) $, and there exists a position $e$ such that $\rho'[i] \models alive$ for $0 \leq i \leq e$ and $\rho'[i] \models \neg alive$ for $i > e$. Thus we have a finite trace $\rho$ such that $\rho \approx \rho'$ with $e$. Therefore, $\rho \models \phi$ holds by Lemma~\ref{lem:sat_equ}.
         \item 
         $GF\alpha$ is falsified such that the environment behaves such as only having $\alpha$ hold for finite times, in which case the fairness assumption is violated, we conclude that $\rho \models \phi$ holds by default. 
     \end{itemize}
     Finally, in order to obtain the winning strategy $g$, we have $g(\lambda) = g'(\lambda)|_{\Y}$, where $\lambda \in (2^\X)^\omega$.
     
     \item
     $\rightarrow: $ Since $\langle \X, \Y, \alpha, \phi \rangle$ is realizable, there is a winning strategy $g: (2^\X)^+ \rightarrow 2^{\Y}$ such that for every trace $\rho$ that follows $g$, either of the following situations happens:
     \begin{itemize}
         \item 
         The environment behaves such as having $\alpha$ hold infinitely often such that $GF\alpha$ is true, then there is $k \geq 0$ such that $\rho^k \models \phi$. Since $alive$ is assigned as an agent variable, we can construct a play $\rho'$ such that $\rho'[i] = \rho[i] \wedge alive$ for $0 \leq i \leq k$ and $\rho'[i] = \rho[i] \wedge \neg alive$ for $i > k$. Thus we have $\rho' \approx \rho$ such that $\rho' \models \psi$ by Lemma~\ref{lem:sat_equ}. Therefore, we conclude that $\rho' \models GF\alpha \rightarrow \psi$ holds.
         \item 
         The environment behaves such as violating the fairness assumption such that only having $\alpha$ hold for finite times, in which case $GF\alpha$ doesn not hold. Therefore, $GF\alpha \rightarrow \psi$ is true.
     \end{itemize}
     Finally, in order to obtain the winning strategy $g'$, we have $g'(\lambda) = g(\lambda) \wedge alive$ if $\phi$ has not been satisfied and $g'(\lambda) = g(\lambda) \wedge \neg alive$ since $\phi$ has been satisfied, where $\lambda \in (2^\X)^\omega$.
 \end{itemize}
\end{proof}

\begin{theorem}
	Let $\phi$ be an \LTLf formula, $\psi$ be the corresponding translated \LTL formula, then stable \LTLf synthesis problem $\langle \X, \Y, \alpha, \phi \rangle$ is realizable if and only if \LTL formula $GF \alpha \rightarrow \psi$ is realizable with respect to $\langle \X, \Y \cup \{alive\} \rangle$.
\end{theorem}

\begin{proof}
	We prove the two directions separately.
	\begin{itemize}
		\item 
		$\leftarrow: $ Since $FG \alpha \rightarrow \psi$ is realizable with respect to $\langle \X, \Y \cup \{alive\} \rangle$, there exists a winning strategy $g': (2^\X)^+ \rightarrow 2^{\Y \cup \{alive\}}$ such that every trace $\rho'$ that follows $g'$ gets $FG \alpha \rightarrow \psi$ hold, therefore, enabling either of the following situations:
		\begin{itemize}
			\item 
			$FG\alpha$ is true such that the environment behaves such as having $\alpha$ eventually hold forever, and $\psi$ holds. Therefore, $\rho' \models alive \wedge (alive~U~(G\neg alive)) $, and there exists a position $e$ such that $\rho'[i] \models alive$ for $0 \leq i \leq e$ and $\rho'[i] \models \neg alive$ for $i > e$. Thus we have a finite trace $\rho$ such that $\rho \approx \rho'$ with $e$. Therefore, $\rho \models \phi$ holds by Lemma~\ref{lem:sat_equ}.
			\item 
			$FG\alpha$ is falsified such that the environment behaves such as having $\neg \alpha$ hold for infinitely many times, in which case the stability assumption is violated, we conclude that $\rho \models \phi$ holds by default. 
		\end{itemize}
		Finally, in order to obtain the winning strategy $g$, we have $g(\lambda) = g'(\lambda)|_{\Y}$, where $\lambda \in (2^\X)^\omega$.
		
		\item
		$\rightarrow: $ Since $\langle \X, \Y, \alpha, \phi \rangle$ is realizable, there is a winning strategy $g: (2^\X)^+ \rightarrow 2^{\Y}$ such that for every trace $\rho$ that follows $g$, either of the following situations happens:
		\begin{itemize}
			\item 
			The environment behaves such as having $\alpha$ eventually hold forever such that $FG\alpha$ is true, then there is $k \geq 0$ such that $\rho^k \models \phi$. Since $alive$ is assigned as an agent variable, we can construct a trace $\rho'$ such that $\rho'[i] = \rho[i] \wedge alive$ for $0 \leq i \leq k$ and $\rho'[i] = \rho[i] \wedge \neg alive$ for $i > k$. Thus we have $\rho' \approx \rho$ such that $\rho' \models \psi$ by Lemma~\ref{lem:sat_equ}. Therefore, we conclude that $\rho' \models FG\alpha \rightarrow \psi$ holds.
			\item 
			The environment behaves such as violating the stability assumption such that having $\neg \alpha$ hold for infinitely many times, in which case $FG\alpha$ does not hold. Therefore, $FG\alpha \rightarrow \psi$ is true.
		\end{itemize}
		Finally, in order to obtain the winning strategy $g'$, we have $g'(\lambda) = g(\lambda) \wedge alive$ if $\phi$ has not been satisfied and $g'(\lambda) = g(\lambda) \wedge \neg alive$ since $\phi$ has been satisfied, where $\lambda \in (2^\X)^\omega$.
	\end{itemize}
\end{proof}

In general, nevertheless the form of the environment assumption $\psi_A$ is, the synthesis problem of \LTLf formula $\phi$ under such assumption can be reduced to standard \LTL synthesis problem $\psi_A \limp \psi$, where $\psi$ is the corresponding translated \LTL formula of $\phi$. The following theorem guarantees this reduction.
\begin{theorem}
	Let $\phi$ be an \LTLf formula, $\psi$ be the corresponding translated \LTL formula, then $\phi$ is realizable with respect to $\langle \X, \Y \rangle$ with assumption $\psi_A$ if and only if \LTL formula $\psi_A \rightarrow \psi$ is realizable with respect to $\langle \X, \Y \cup \{alive\} \rangle$.
\end{theorem}

\begin{proof}
	We prove the two directions separately.
	\begin{itemize}
		\item 
		$\leftarrow: $ Since $\psi_A \rightarrow \psi$ is realizable with respect to $\langle \X, \Y \cup \{alive\} \rangle$, there exists a winning strategy $g': (2^\X)^+ \rightarrow 2^{\Y \cup \{alive\}}$ such that every trace $\rho'$ that follows $g'$ gets $\psi_A \rightarrow \psi$ hold, therefore, enabling either of the following situations:
		\begin{itemize}
			\item 
			$\psi_A$ is true such that the environment behaves such as having $\psi_A$ hold, and $\psi$ holds. Therefore, $\rho' \models alive \wedge (alive~U~(G\neg alive)) $, and there exists a position $e$ such that $\rho'[i] \models alive$ for $0 \leq i \leq e$ and $\rho'[i] \models \neg alive$ for $i > e$. Thus we have a finite trace $\rho$ such that $\rho \approx \rho'$ with $e$. Therefore, $\rho \models \phi$ holds by Lemma~\ref{lem:sat_equ}.
			\item 
			$\psi_A$ is falsified such that the environment behaves such as having assumption $\psi_A$ get violated, we conclude that $\rho \models \phi$ holds by default. 
		\end{itemize}
		Finally, in order to obtain the winning strategy $g$, we have $g(\lambda) = g'(\lambda)|_{\Y}$, where $\lambda \in (2^\X)^\omega$.
		
		\item
		$\rightarrow: $ Since $\phi$ is realizable with respect to $\langle \X, \Y \rangle$ under assumption $\psi_A$, there is a winning strategy $g: (2^\X)^+ \rightarrow 2^{\Y}$ such that for every trace $\rho$ that follows $g$, either of the following situations happens:
		\begin{itemize}
			\item 
			The environment behaves such as having $\psi_A$ hold, then there is $k \geq 0$ such that $\rho^k \models \phi$. Since $alive$ is assigned as an agent variable, we can construct a trace $\rho'$ such that $\rho'[i] = \rho[i] \wedge alive$ for $0 \leq i \leq k$ and $\rho'[i] = \rho[i] \wedge \neg alive$ for $i > k$. Thus we have $\rho' \approx \rho$ such that $\rho' \models \psi$ by Lemma~\ref{lem:sat_equ}. Therefore, we conclude that $\rho' \models \psi_A \rightarrow \psi$ holds.
			\item 
			The environment behaves such as violating the assumption $\psi_A$, in which case $\psi_A \rightarrow \psi$ is true.
		\end{itemize}
		Finally, in order to obtain the winning strategy $g'$, we have $g'(\lambda) = g(\lambda) \wedge alive$ if $\phi$ has not been satisfied and $g'(\lambda) = g(\lambda) \wedge \neg alive$ since $\phi$ has been satisfied, where $\lambda \in (2^\X)^\omega$.
	\end{itemize}
\end{proof}

\end{document}